\PassOptionsToPackage{table}{xcolor}
\documentclass{article}
\usepackage{iclr2027_conference,times}

\usepackage{amsmath,amsfonts,bm, amsthm}
\usepackage{thmtools}
\usepackage{thm-restate}
\newtheorem{theorem}{Theorem}[section]

\declaretheorem[sibling=theorem,name=Corollary]{corollary}

\def\eqref#1{equation~\ref{#1}}

\def\1{\bm{1}}

\def\vx{{\bm{x}}}

\def\mA{{\bm{A}}}

\def\mE{{\bm{E}}}

\def\mJ{{\bm{J}}}

\def\mW{{\bm{W}}}
\def\mX{{\bm{X}}}

\DeclareMathAlphabet{\mathsfit}{\encodingdefault}{\sfdefault}{m}{sl}
\SetMathAlphabet{\mathsfit}{bold}{\encodingdefault}{\sfdefault}{bx}{n}

\usepackage{hyperref}
\usepackage{url}
\usepackage[table]{xcolor}
\usepackage{booktabs}

\usepackage{cleveref}
\usepackage{thmtools}
\usepackage{thm-restate}
\usepackage{multicol}

\usepackage{multirow}
\usepackage{graphicx}
\usepackage{amsmath}
\usepackage{wrapfig}
\usepackage{subcaption}

\newcommand{\model}{SGT}
\newcommand{\modelext}{S-Graph Transformer}
\newcommand{\ie}{i.e., }
\newcommand{\eg}{e.g., }

\def\cayley{\mathrm{cay}}

\title{Stable Transformers for Graph Generation}

\author{Luca Miglior\thanks{Equal contribution.}\,\,\,\thanks{Correspondence: luca.miglior@phd.unipi.it, alessio.gravina@di.unipi.it.} \quad
Alessio Gravina\footnotemark[1] \quad
Davide Bacciu \\
Department of Computer Science, University of Pisa, Italy}

\iclrfinalcopy
\begin{document}

\maketitle

\begin{abstract}
Graph generative models increasingly rely on Graph Transformers (GT) to capture complex dependencies among nodes and edges. While deeper architectures should provide greater expressive capacity and a broader receptive field, their effectiveness can decline with depth: repeated self-attention progressively contracts node representations, impeding information flow and gradient propagation. We analyse this phenomenon from a dynamical systems perspective, focusing on how the denoiser's spectral dynamics affect graph generation. We show that standard GT denoisers become increasingly dissipative as depth grows, leading to vanishing gradients and representation collapse. 
To isolate the effect of these dynamics, we construct  a permutation-equivariant GT with inherently stable, non-dissipative transport. We also introduce a damping mechanism that continuously interpolates between non-dissipative and increasingly contractive regimes, enabling a direct assessment of how dissipation influences generation. Experiments on synthetic and molecular graph generation benchmarks show that the gap between these regimes widens with depth: non-dissipative dynamics preserve representation diversity and gradient flow, sustaining strong generative performance, whereas greater contraction progressively impairs it. These findings identify the denoiser's dynamical regime as a key design factor for deep graph generative models.
\end{abstract}

\section{Introduction}\label{sec:introduction}
Learning to generate graphs is central to a wide range of problems in which both the entities and their relations must be modeled jointly, such as in biology and life science \citep{Li2026}. Recent graph generative models have advanced substantially, with diffusion- and flow-matching-based approaches emerging as particularly effective frameworks \citep{vignac2023digress, qin2025defog, liu2024graphdiffusiontransformersmulticonditional, jo2022scorebasedgenerativemodelinggraphs, hou2024improvingmoleculargraphgeneration, eijkelboom2025variationalflowmatchinggraph, jang2024graphgenerationk2trees, carballo-castro2026generating, zhao2024pard, luo2026simgfm}. A common trait of these methods is their increasing reliance on expressive Graph Transformer (GT) models \citep{graphtransformer,dwivedi2021generalization}, making the architecture's properties crucial to generation. Greater depth is especially appealing because it increases expressive capacity, enriches representations, and allows information to propagate farther across the graph, potentially capturing longer-range dependencies.
However, a wider receptive field does not necessarily ensure effective long-range communication.  
Although stacking layers progressively enlarges the receptive field, the influence of distant nodes can rapidly fade because of over-smoothing \citep{cai2020note, oono2020graph, rusch2023survey} and over-squashing \citep{alon2021oversquashing, topping2022understanding, diGiovanniOversquashing, mishayev2025shortrange}, both closely linked to vanishing gradients \citep{gravina_gnn_ssm}.
Similar effects have been observed in Transformers \citep{vaswani2023attentionneed}, where repeated self-attention can progressively reduce representation diversity and drive token representations toward low-rank or collapsed states \citep{attention_is_not_all_you_need, transformer_rank_collapse}.
In graph generation, tokens represent nodes whose distinct identities are essential to reconstructing graph structure. Increasing depth may therefore weaken the very information that additional layers are meant to propagate. This leads to our central question: \emph{to what extent is the behavior of a deep Graph Transformer determined by the dynamical regime underpinning its layers?}

We address this question from a dynamical-systems perspective and show that standard self-attention contracts representations along the node axis. Repeated composition consequently attenuates the information encoded in differences between nodes. 
To the best of our knowledge, this is the first work to directly connect the dissipative dynamics of Graph Transformer denoisers to graph generation quality.
Our analysis identifies stable and non-dissipative node dynamics as a principled design criterion for graph generative models which preserve information across layers, allowing graph generative models to exploit depth more effectively. \begin{wrapfigure}{r}{0.6\linewidth}
    \centering
    \includegraphics[width=\linewidth]{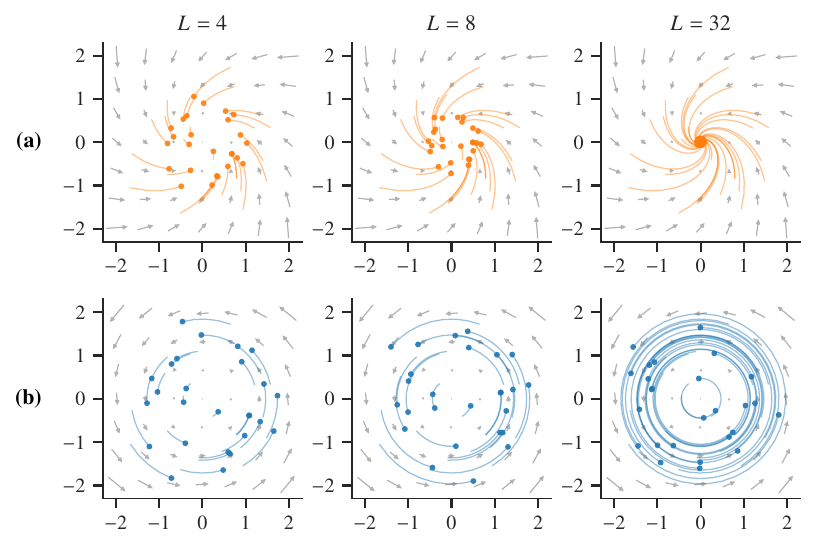}
    \caption{\small Illustration of node-state trajectories across depth $L$. \textbf{(a)} Under a standard Graph Transformer denoiser, node states spiral into a single point and node-specific information is lost. \textbf{(b)} Under non-dissipative orthogonal transport (ours), node states rotate while keeping their norms and remain distinct.}
    \label{fig:teaser}
\end{wrapfigure}
Building on this insight, we replace the contractive attention map with a Cayley-based orthogonal node-mixing operator. Acting directly on the node dimension, the operator is guaranteed to be orthogonal by construction and prevents systematic contraction across layers (see \Cref{fig:teaser}). This yields a simple drop-in replacement for standard self-attention in existing GT denoisers, leaving the surrounding generative framework and procedure unchanged. Our approach also interpolates continuously between non-dissipative and increasingly contractive regimes, enabling us to isolate the role of vanishing gradients in deep GTs. By varying contraction strength, we track how information and node representations evolve with depth, while the pure non-dissipative regime allows us to assess whether preserving signal propagation sustains long-range interactions. Our experiments show that these differences become increasingly pronounced as depth grows: non-dissipative dynamics retain strong generative performance, whereas contractive dynamics progressively degrade it.

\textbf{Contributions.}
Our main contributions are threefold. \textbf{(i)} In \Cref{sec:theory}, we provide, to the best of our knowledge, the first study connecting the dissipative dynamics of Graph Transformer denoisers to graph generation quality. We show that standard self-attention progressively contracts node representations with depth, leading to the loss of node-specific information and vanishing gradients. \textbf{(ii)} In \Cref{sec:method}, we introduce a simple drop-in modification of the Graph Transformer node-mixing operator based on the Cayley transform. The resulting transport is non-dissipative by construction, preserves permutation equivariance, and can be integrated into existing GT-based generative models without modifying their training objective or sampling procedure. \textbf{(iii)} Finally, in \Cref{sec:experiments}, we systematically control the amount of dissipation in discrete flow-matching graph generation and establish its effect on generation quality. The advantage of non-dissipative dynamics grows with depth: contractive denoisers degrade and eventually collapse as layers are added, while non-dissipative ones continue to benefit from them.

 \section{Background and Related Work}
\label{sec:background_related_work}

\subsection{Discrete Flow Matching for Graph Generation}
Score- and diffusion-based models have recently become prominent approaches to graph generation through iterative denoising of node and edge states. 
Such methods differ in how this generative process is instantiated: GDSS~\citep{jo2022scorebasedgenerativemodelinggraphs} evolves continuous graph representations through score-based stochastic dynamics, whereas DiGress~\citep{vignac2023digress} performs diffusion directly over categorical node and edge attributes. Flow Matching (FM)~\citep{lipman2023flowmatchinggenerativemodeling} instead learns a transport dynamics along a prescribed probability path, with Discrete Flow Matching~\citep{gat2024discreteflowmatching,campbell2024generativeflowsdiscretestatespaces} extending this principle to discrete state spaces. Despite these differences, recent graph generative models commonly rely on GT to parameterize the evolution of node and edge states, making multi-head attention a shared architectural component across otherwise distinct generative formulations.
Within this landscape, we adopt DeFoG~\citep{qin2025defog} because of its robustness, efficiency and compatibility with discrete graph structures. 

Let $\mathcal{G}_t=(\mathcal{V},\mX_t,\mE_t)$ denote the graph state at time $t\in[0,1]$, where $\mathcal{V}$ contains $N$ nodes and $\mX_t,\mE_t$ represent categorical node and edge states. Following DeFoG, we define a factorized conditional path $p_{t|1}(\mathcal{G}_t\mid\mathcal{G}_1)$ that describes intermediate states conditioned on a clean target graph $\mathcal{G}_1$:
\begin{equation}
    p_{t|1}(\mathcal{G}_t\mid\mathcal{G}_1)
    =
    \prod_{i=1}^{N}
    p^X_{t|1}(x_t^i\mid x_1^i)
    \prod_{1\leq i<j\leq N}
    p^E_{t|1}(e_t^{ij}\mid e_1^{ij}),
\end{equation}
where each factor interpolates between the corresponding source distribution and the clean state:
\begin{equation}
    p^Z_{t|1}(z_t\mid z_1)
    = (1-t)p_0^Z(z_t)+t\,\delta_{z_1}(z_t),
    \qquad Z\in\{X,E\}.
\end{equation}
Here, $\delta_{z_1}$ denotes a point mass at $z_1$. Averaging the conditional path over $\mathcal{G}_1\sim p_{\mathrm{data}}$ yields a marginal path $p_t$ that connects source noise at $t=0$ to the data distribution at $t=1$.

Generation follows this path through a continuous-time Markov chain (CTMC), with transition rates governing jumps between categorical states. For each node or edge variable,
\begin{equation}
    \Pr(z_{t+h}=z'\mid\mathcal{G}_t)
    = \delta_{z_t}(z')
    + h R_t^\theta(z_t,z';\mathcal{G}_t)+o(h).
\end{equation}
The rates $R_t^\theta$ are computed by averaging the conditional transition rates over the predicted clean-state marginal
$p^\theta_{1|t}(z_1\mid\mathcal{G}_t)$.
The denoiser is trained with standard cross-entropy to predict these marginals for all nodes and edges, conditioned on the  current graph. We analyse its architecture as a dynamical system across network depth, while retaining DeFoG’s training objective and CTMC generation framework as the foundation of our contribution.

Literature methods, such as~\cite{vignac2023digress,huang2023conditionaldiffusionbaseddiscrete,jo2024graphgenerationdiffusionmixture,qin2025defog,eijkelboom2025variationalflowmatchinggraph,luo2026simgfm}, typically parametrize the denoising network as a GT with standard multi-head self-attention \citep{vaswani2023attentionneed}. In each layer, head $h$ propagates information between nodes through an attention matrix $\mA^{(\ell,h)}\in\mathbb{R}^{N\times N}$.
For node representations $\mX^{(\ell)}\in\mathbb{R}^{N\times d}$, the layer computes
\begin{equation}\label{eq:graph_transformer}
    \mX^{(\ell+1)} = \text{FF}(\text{LayerNorm}(
    \mX^{(\ell)} + \bigoplus
    \mA^{(\ell,h)}
    \mX^{(\ell)}
    \mW_V^{(\ell,h)})),
\end{equation}
where $\mW_V^{(\ell,h)}\in\mathbb{R}^{d\times d}$ is the value projection matrix for layer $\ell$ and head $h$, and $\bigoplus$ denotes head aggregation (\eg concatenation). An output projection combines the head outputs before the residual connections, normalization, and feed-forward blocks. Softmax normalization makes every attention matrix row-stochastic:
$\mA^{(\ell,h)}\geq 0$ entrywise and
$\mA^{(\ell,h)}\mathbf{1}=\mathbf{1}$.
Therefore, each head forms convex combinations of node values. This averaging behaviour motivates our analysis of how repeated attention transforms representations across network depth.
\subsection{Effective Information Propagation}\label{sec:effective_information}
Several studies \citep{HaberRuthotto2017, neuralODE,chang2018antisymmetricrnn, GDE, gravina_adgn, gravina_gnn_ssm} analyse information propagation by interpreting neural networks as discretized dynamical systems governed by differential equations. 
From this perspective, network layers correspond to successive steps of a numerical scheme for solving the differential equation: each layer transformation represents one step in the evolution of the underlying continuous dynamical system.
Information preservation and propagation are governed by the system Jacobian \citep{Ascher1998, HaberRuthotto2017}:
\begin{equation}
    \mJ = \prod_{\ell=1}^L \frac{\partial \vx^{(\ell)}}{\partial \vx^{(\ell-1)}} = \prod_{\ell=1}^L \mJ_\ell
\end{equation}
where $L$ is the number of neural layers and $\vx^{(\ell)}$ is the state vector at layer $\ell$.
The singular values of the Jacobian quantify directional changes in representation space. Specifically, values below one attenuate perturbations and promote vanishing gradients, whereas values above one amplify them and promote exploding gradients. Across many layers, even mild contraction or expansion can therefore produce substantial signal attenuation or amplification.

We distinguish three regimes. Dynamics are \textit{dissipative} when representations progressively contract and information decays through composition, typically because eigenvalues lie strictly inside the unit circle; smaller magnitude implies faster modes decay.  Dynamics are \textit{unstable} when eigenvalues lie outside the unit circle, amplifying perturbations across layers. Stability therefore requires eigenvalues to remain within the unit circle, but does not by itself guarantee information preservation: a stable system may still be strongly dissipative and rapidly forget its input as depth increases. Finally, \textit{non-dissipative} dynamics occupies the boundary regime in which the relevant eigenvalues lie on the unit circle, preventing asymptotic information decay.

Vanishing gradients have been extensively studied in sequence modeling \citep{bengio1994learning, hochreiter1997long, pascanu2013difficulty}, motivating architectures that preserve information over long sequences \citep{arjovsky2016unitary, henaff2016, orvieto2023resurrecting, gu2021, gu2023mamba}. More recently, analogous principles have been applied to graph learning \citep{gravina_adgn, gravina_gnn_ssm}, where contractive spectral structures can cause severe gradient vanishing and information loss. These findings show that effective long-range modeling depends on preserving signal strength through non-dissipative dynamics. However, whether similar spectral phenomena arise in graph generative models, and how they affect the denoising dynamics with increasing depth, remains largely unexplored.

 \subsection{Deep Graph Transformers Suffer from Rank-Collapse}\label{sec:theory}
In this section, we argue why the class of GTs commonly used as denoisers in graph generative models are prone to vanishing gradients and relate this behavior to the spectral contraction of their dynamics. 

Recent works \citep{attention_is_not_all_you_need, transformer_rank_collapse, saada2025mind} have studied signal propagation in Transformers through \emph{rank collapse}, where token representations become increasingly aligned with depth. In particular, \citet{attention_is_not_all_you_need} showed that pure self-attention networks converge to rank-one representations. This phenomenon affects not only forward dynamics but also gradient propagation. \citet{transformer_rank_collapse} showed that increasing token alignment causes vanishing gradients in the query and key parameters. More recently, \citet{saada2025mind} provided a spectral characterization of this behavior, showing that the spectral gap induced by softmax attention promotes rank collapse and degrades gradient propagation. 
These findings link the mechanisms driving representation collapse in deep Transformers to their ability to propagate informative signals and gradients effectively.
\begin{figure}[h]
    \centering

    \includegraphics[width=\textwidth]{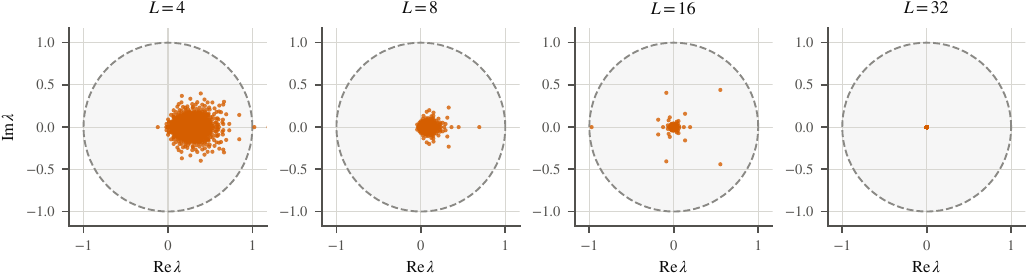}
\caption{Jacobian spectra of the GT by \cite{vignac2023digress} for different number of layers $L$.}
    \label{fig:baseline_spectra}
\end{figure}
In Graph Transformers, where tokens represent nodes \citep{graphtransformer,dwivedi2021generalization}, this degeneration directly impairs information propagation across the graph. As node representations progressively collapse with increasing depth, the model becomes less sensitive to node-specific information, making the dynamics increasingly dissipative. In this scenario, \citet{gravina_gnn_ssm} show that the contractive nature of graph propagation can jointly cause feature collapse and vanishing gradients, limiting both the influence of distant nodes and modeling of node dependencies. Moreover, \citet{zhao2023more} show that deeper GTs face an attention-capacity bottleneck that increasingly restricts their ability to identify and propagate information from relevant graph substructures. Together, these results indicate that greater depth can induce increasingly contractive dynamics that hinder effective information propagation in GTs.

We empirically examine this behavior through the Jacobian spectrum in \Cref{fig:baseline_spectra} for the GT
introduced by \citet{vignac2023digress}, which is representative of the Transformer-based denoisers used in recent graph generative models, \eg \cite{qin2025defog,eijkelboom2025variationalflowmatchinggraph, luo2026simgfm}.
Shallow models retain Jacobian eigenvalues of substantial magnitude across layers, whereas increasing depth progressively drives them toward zero, indicating stronger contraction and vanishing gradients. 
As a result, deep GT denoisers struggle to propagate gradients and preserving node-specific information.

 \section{Controlling Graph Transformer Dynamics}\label{sec:method}

The preceding analysis characterized deep Transformer dynamics in terms of information propagation on graphs. We now build on this perspective to design a Graph Transformer with stable, non-dissipative propagation. 
We enforce a near-orthogonal Jacobian, with singular values close to one, to preserve the norms of forward and backward gradients across discrete composition of layers. The architecture must also retain graph-dependent interactions and permutation equivariance, ensuring that propagation depends on the graph structure but not on node ordering. These principles define our \modelext{} (\model{}), introduced next.

\begin{figure}
    \centering

    \begin{subfigure}[t]{0.48\linewidth}
        \centering
        \includegraphics[width=0.755\linewidth]{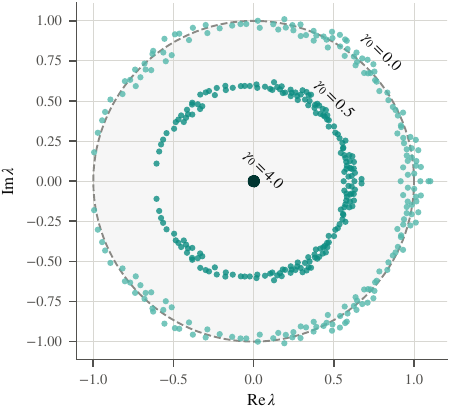}
        \caption{\small
        Jacobian eigenvalue spectra for different $\gamma$.}
        \label{fig:method_gamma_spectrum}
    \end{subfigure}
    \hfill
    \begin{subfigure}[t]{0.48\linewidth}
        \centering
        \includegraphics[width=\linewidth]
        {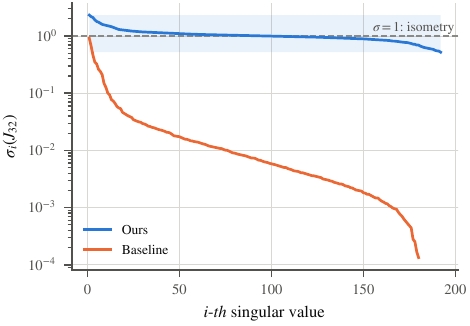}
        \caption{\small
        Jacobian singular values at $L=32$.}
        \label{fig:method_initial_spectrum}
    \end{subfigure}

    \caption{Spectral behavior of \model{}.
    \textbf{(a)} Jacobian eigenvalue spectra of \model{} with $L=8$ for different
    values of $\gamma$. Increasing $\gamma$ progressively moves the spectrum
    from the unit circle into increasingly contractive regimes.
    \textbf{(b)} Jacobian singular values at $L=32$ for \model{} and GT proposed by \cite{qin2025defog}.}
    \label{fig:method_spectral_analysis}
\end{figure}

\subsection{\modelext{}}\label{sec:model_description}
The core idea is to replace the contractive node mixing of the standard GT with a graph-dependent orthogonal propagation operator. We construct it from a skew-symmetric attention-score matrix, whose Cayley transform is guaranteed to be orthogonal \citep{Cayley1846, helfrich2018orthogonalrecurrentneuralnetworks}, thereby enforcing non-dissipative dynamics.

Let $\mathbf{X}^{(\ell)}\in\mathbb{R}^{N\times d}$ denote the node
representations at layer $\ell$, with edge representations
$\mathbf{E}^{(\ell)}$ and a conditioning vector $\mathbf{y}^{(\ell)}$
containing the time-level and graph-level conditioning information described in \Cref{sec:background_related_work}, and described in detail in Appendix \ref{app:additional_experiments}. We split the node channels into $H$ heads,
writing $\mathbf{X}^{(\ell,h)}\in\mathbb{R}^{N\times d_h}$, where $d_h=d/H$.
All equations below are defined over the $N$ active nodes.

Let $\mathbf{S}^{(\ell,h)}$ be the attention-score matrix  at layer $\ell$ and head $h$, obtained from the query-key interactions before softmax normalization. We define the skew-symmetric attention-score matrix\footnote{A matrix $\mathbf{A}\in\mathbb{R}^{d \times d}$ is skew-symmetric if $\mathbf{A}^\top = -\mathbf{A}$.} $\mathbf{K}^{(\ell,h)}$ as
\begin{equation}
    \mathbf{K}^{(\ell,h)}
    = \frac{\varepsilon_{\ell,h}(\mathbf{y}^{(\ell)})}{2\sqrt{N}}
      \left(\mathbf{S}^{(\ell,h)}-
            \mathbf{S}^{(\ell,h)\top}\right),
    \label{eq:node_generator}
\end{equation}
where $\varepsilon_{\ell,h}$ is a learned gate applied to the conditioning vector, and the scaling factor adjusts its magnitude to the graph size. We then apply the Cayley transform
\begin{equation}
\mathbf{R}^{(\ell,h)}
= \cayley(\mathbf{K}^{(\ell,h)})
:= \left(\mathbf{I}-\tfrac12\mathbf{K}^{(\ell,h)}\right)^{-1}
\left(\mathbf{I}+\tfrac12\mathbf{K}^{(\ell,h)}\right).
\label{eq:cayley_node}
\end{equation}
This transform maps $\mathbf{K}^{(\ell,h)}$ to the orthogonal attention-score matrix $\mathbf{R}^{(\ell,h)}$\footnote{Thus, \(\mathbf{R}^{(\ell,h)\top}\mathbf{R}^{(\ell,h)}=\mathbf{I}\).}, which acts as the node-propagation operator. The output of head $h$ is 
\begin{equation}
    \mathbf{Z}^{(\ell,h)} = \mathbf{R}^{(\ell,h)}\mathbf{X}^{(\ell,h)}\mathbf{C}^{(\ell,h)},
\end{equation}
where $\mathbf{C}^{(\ell,h)}= \cayley(\mathbf{W}^{(\ell,h)})$ is an orthogonal channel mixing matrix, obtained by applying the Cayley transform to the learnable skew-symmetric weight matrix $\mathbf{W}^{(\ell,h)}$. The head outputs are then combined as
\begin{align}
    \tilde{\mathbf{X}}^{(\ell)} &= \operatorname{Concat}_{h=1}^{H}
       \bigl(\mathbf{Z}^{(\ell,h)}\bigr)\mathbf{O}^{(\ell)}
       +{f}_{\ell}(\mathbf{y}^{(\ell)})\label{eq:orthogonal_transport}\\
    \mathbf{X}^{(\ell+1)}
    &=\tilde{\mathbf{X}}^{(\ell)}
     +\alpha_{\ell}\,
       \operatorname{FFN}_{\ell}
       \bigl(\operatorname{LN}(\tilde{\mathbf{X}}^{(\ell)})\bigr)
    \label{eq:orthogonal_ffn}
\end{align}
where $\mathbf{O}^{(\ell)}$ is either the identity or a Cayley-parametrized orthogonal matrix that mixes channels across heads, $\alpha_{\ell}$ is a learned scalar, and ${f}_{\ell}$ linearly transforms the conditioning vector. A  final node LayerNorm is applied before the output Feed Forward Neural Network. 

\subsection{Guarantees of Orthogonal Transport}\label{sec:guarantees}
We now establish the main properties of the orthogonal transport defined above.

\textbf{Non-dissipative propagation.} The Cayley-based construction prevents information from progressively decaying across layers. Because both node propagation and channel-mixing are orthogonal, the transport map preserves signal norms throughout the network. The following result formalizes this property and characterizes the Jacobian spectrum.

\begin{restatable}[Isometric node transport]{theorem}{orthogonaltransport}
\label{thm:orthogonal_transport}
Consider the attention submap in \Cref{eq:orthogonal_transport} restricted to active nodes, with edge and global inputs fixed and the attention scores used in the attention branch. Its node Jacobian $\mathbf{J}_{\ell}$ is orthogonal:
\begin{equation}
    \mathbf{J}_{\ell}^{\top}\mathbf{J}_{\ell}=\mathbf{I},
    \qquad \sigma_i(\mathbf{J}_{\ell})=1,
    \qquad |\lambda_i(\mathbf{J}_{\ell})|=1.
    \label{eq:orthogonal_jacobian}
\end{equation}
Therefore, products of these conditional transport Jacobians preserve the norms of perturbations and backpropagated gradients at every depth.
\end{restatable}

The proof is in Appendix~\ref{app:transport_proof}. By \Cref{sec:effective_information}, the proposed transport lies in the non-dissipative regime and therefore preserves information and gradient propagation across depth. \Cref{fig:method_spectral_analysis} shows the resulting Jacobian spectrum.

\textbf{Permutation equivariance.}
Orthogonalizing the node propagation must preserve a fundamental property
of graph architectures: equivariance to node relabeling. Cayley-based construction, retains, in fact, this property. Intuitively, the result follows from the fact that node relabeling conjugates the attention scores, the skew-symmetric generator, and the corresponding Cayley map by the same permutation matrix, while the remaining channel-wise and node-wise operations commute with the relabeling. Orthogonalizing the propagation operator therefore preserves the GT's permutation equivariance. We show the complete proof in Appendix~\ref{app:permutation_equivariance}. 

\subsection{Controlling \model{} Dynamics}
To assess how dissipation affects generation, we introduce a mechanism that continuously interpolates between non-dissipative and increasingly contractive regimes while keeping the rest of the architecture fixed.
Specifically, we define a nonnegative damping parameter $\gamma$ that shifts the skew-symmetric matrix in \Cref{eq:cayley_node} before the Cayley transform:
\begin{equation}
    \mathbf{R}_{\gamma}^{(\ell,h)}
    = \cayley
      \bigl(\mathbf{K}^{(\ell,h)}-\gamma\mathbf{I}\bigr),
    \qquad
    \gamma=\frac{\gamma_0}{L},
    \label{eq:damped_cayley}
\end{equation}
where $L$ is the number of layers and $\gamma_0\geq 0$ sets the total dissipation across depth. Thus, $\gamma_0=0$ recovers the orthogonal transport above, whereas larger values progressively make the dynamics more contractive.  \Cref{fig:method_gamma_spectrum} illustrates the corresponding change in the Jacobian spectrum.

The spectral effect of damping follows directly from the Cayley transform. Without damping, the eigenvalues of $\mathbf{K}^{(\ell,h)}$ map to the unit circle; shifting by $-\gamma\mathbf{I}$ moves the spectrum into the left half-plane, placing the Cayley transform inside the unit circle. The following result characterizes this transition.

\begin{restatable}[Controlled dissipation]{proposition}{controlleddissipation}
\label{prop:controlled_dissipation}
Let $\mathbf{K}$ be the skew-symmetric attention score matrix from \Cref{eq:node_generator}. For each eigenvalue $\mathrm{i}\omega$ of $\mathbf{K}$, the corresponding eigenvalue of  $\,\mathbf{R}_{\gamma}=\cayley
(\mathbf{K}-\gamma\mathbf{I})$ is
\begin{equation}
    r_{\gamma}(\omega)
    = \frac{1-\gamma/2+\mathrm{i}\omega/2}
           {1+\gamma/2-\mathrm{i}\omega/2},
    \qquad
    |r_{\gamma}(\omega)|^2
    = \frac{(1-\gamma/2)^2+\omega^2/4}
           {(1+\gamma/2)^2+\omega^2/4}.
    \label{eq:damped_spectrum}
\end{equation}
The matrix $\mathbf{R}_{\gamma}$ is normal, and its singular values are $|r_{\gamma}(\omega)|$. 
\end{restatable}
The proof is given in Appendix~\ref{app:damping_proof}. Proposition~\ref{prop:controlled_dissipation} shows that $\gamma_0$ directly controls the transport operator's spectral contraction, enabling systematic variation of attenuation across depth without changing the rest of the architecture.

\section{How Does Dissipativity Affect Graph Generation?}\label{sec:experiments}
\begin{wrapfigure}{r}{0.28\linewidth}
    \centering
    \vspace{-2em}
    \includegraphics[width=\linewidth]{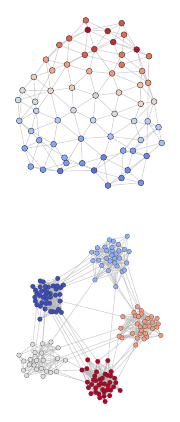}
    \caption{\small Graphs generated by \model{} on Planar (top) and SBM
    (bottom)}
    \label{fig:planar_sbm_pair}
    \vspace{-1em}
\end{wrapfigure}
Our analysis so far identifies dissipation as a constraint on deep GT-based denoisers. We now shift from comparing architectures to a more direct question: \emph{how does the denoiser's dynamical regime affect graph generation?} 
We address this question by varying dissipation and depth jointly while keeping the surrounding flow-matching framework unchanged.

\textbf{Setup.}
We empirically study how dissipativity and denoiser depth influence graph generation, following established experimental
protocols~\citep{luo2026simgfm, qin2025defog, siraudin2024comethcontinuoustimediscretestategraph, vignac2023digress}.
We use \model{}, introduced in \Cref{sec:model_description}, as the
denoising backbone within DeFoG's discrete flow matching framework,
retaining its training objective and CTMC sampling procedure.
For a fair comparison, we retrain all DeFoG baselines reported below using the public implementation and best-performing hyperparameters. We evaluate DeFoG and \model{} with progressively deeper GTs. Using DeFoG as a strong Transformer baseline lets us assess how each method benefits from greater depth and isolate the contribution of orthogonal node transport. Results for the remaining baselines are taken from \citet{qin2025defog}. Further baseline details are in Appendix~\ref{app:baselines}.

\textbf{Datasets and Metrics.}
We evaluate structural validity and distributional agreement on two synthetic benchmarks and two molecular generation tasks, using the metrics of \citet{qin2025defog}.
For the Planar and SBM synthetic datasets \citep{sbmplanar}, we report V.U.N, the fraction of simultaneously
valid, unique, and novel graphs, and Ratio, the average
generated-to-test MMD across graph statistics, normalized by the
corresponding mean training-to-test MMD.
For molecular generation, we use ZINC250k~\citep{irwinZINCFreeTool2012}
and MOSES~\citep{moses}, measuring validity, uniqueness, and
distributional agreement via Frechet ChemNet Distance (FCD).
Additional metrics for both molecular benchmarks are in Appendix~\ref{app:additinal_metrics}.
\par

\subsection{Synthetic Graph Generation}\label{sec:synthetic_experiments}
\begin{table}[t]
\centering
\caption{Graph generation performance on Planar and SBM.
    For our model, we report results at each depth $L$.
    Metrics are computed on 40 generated graphs.
    Higher V.U.N. is better and lower ratio is better. Best in bold, second best underlined.}
\label{tab:synthetic_main}
\small
\begin{tabular}{l cc cc}
    \toprule
    & \multicolumn{2}{c}{Planar} & \multicolumn{2}{c}{SBM} \\
    \cmidrule(lr){2-3}\cmidrule(lr){4-5}
    Model & V.U.N.$\uparrow$ & Ratio$\downarrow$ & V.U.N.$\uparrow$ & Ratio$\downarrow$ \\
    \midrule
    Train set & 100.0 & 1.0 & 85.9 & 1.0 \\
    \midrule
    GraphRNN & 0.0 & 490.2 & 5.0 & 14.7 \\
    GRAN & 0.0 & 2.0 & 25.0 & 9.7 \\
    SPECTRE & 25.0 & 3.0 & 52.5 & 2.2 \\
    EDGE & 0.0 & 431.4 & 0.0 & 51.4 \\
    BwR (EDP-GNN) & 0.0 & 251.9 & 7.5 & 38.6 \\
    BiGG & 5.0 & 16.0 & 10.0 & 11.9 \\
    GraphGen & 7.5 & 210.3 & 5.0 & 48.8 \\
    HSpectre & 95.0 & 2.1 & 75.0 & 10.5 \\
    \midrule
    DiGress & 77.5 & 5.1 & 60.0 & 1.7 \\
    DisCo & 83.6 & --- & 66.2 & --- \\
    Cometh & \underline{99.5} & --- & 75.0 & --- \\
    GruM & 90.0 & \underline{1.8} & 85.0 & \underline{1.1} \\
    CatFlow & 80.0 & --- & 85.0 & --- \\
    DeFoG (retrained $L=8$) & 95.5 & 2.41 & 89.7 & 2.12 \\
    DeFoG (retrained $L=16$) & 90.0 & 3.13 & 84.9 & 2.41 \\
    DeFoG (retrained $L=32$) & 0.0 & 106.12 & 0.0 & 23.98 \\
    \midrule
    \rowcolor{gray!15}
    \model{} ($L=8$) & 90.0 & 10.22 & \textbf{95.0} & 1.38 \\
    \rowcolor{gray!15}
    \model{} ($L=16$) & \textbf{100.0} & 2.84 & \textbf{95.0} & 1.49 \\
    \rowcolor{gray!15}
    \model{} ($L=32$) & \textbf{100.0} & \textbf{1.32} & \underline{92.5} & \textbf{1.05} \\
    \bottomrule
\end{tabular}
\end{table}
 Table~\ref{tab:synthetic_main} shows that the advantage of \model{} grows with depth.
At $L=8$, the depth originally employed by \cite{qin2025defog}, results are mixed: \model{} improves on retrained DeFoG on SBM
(V.U.N.\ $95.0\%$ versus $89.7\%$, Ratio $1.38$ versus $2.12$) but trails it on Planar
($90.0\%$ versus $95.5\%$, Ratio $10.22$ versus $2.41$).
From $L=16$ onward, \model{} outperforms DeFoG on both datasets, and at $L=32$ DeFoG collapses
($0.0\%$ V.U.N.) while \model{} attains its lowest Ratio on both benchmarks.

Planar shows the clearest depth dependence: increasing $L$ from $8$
to $32$ drives Ratio from $10.22$ down to $1.32$, outperforming DeFoG's results. Notably, V.U.N. and Ratio improve on different depth scales: V.U.N. reaches $100.0\%$ at $L=16$ and remains saturated, whereas Ratio more than halves across the remaining layers. Thus greater depth continues to sharpen the generated graph distribution even after V.U.N. saturates.
SBM likewise shows a consistent overall improvement in Ratio.

At $L=32$, Ratio reaches its minimum of $1.05$, while V.U.N.
remains above DeFoG's reference.
To isolate the effect of the dynamical regime, Appendix~\ref{app:ablation_gamma}
reports an ablation that progressively increases \model{}'s contraction through $\gamma_0$. As the dynamics become increasingly contractive and approach the regime observed for DeFoG (and related denoisers in \Cref{sec:theory}),
generation performance consistently declines. This offers independent 
evidence that deeper denoisers improve performance by maintaining 
propagation non-dissipative, rather than from depth alone.
We additionally show samples from \model{} final trained models on synthetic benchmarks in \Cref{fig:planar_sbm_pair}. 

\subsection{Molecular Graph Generation}

\begin{table}[t]
\centering
\caption{Molecule generation on ZINC ($10{,}000$ samples) and MOSES ($25{,}000$ samples); DeFoG is retrained by us. MOSES FCD is measured against the scaffold-split test set (TestSF), as for all baselines; Test-split FCD is in \Cref{tab:moses_full}. Best in bold, second best underlined.}
\label{tab:molecular_generation}
\footnotesize
\setlength{\tabcolsep}{3pt}
\begin{minipage}[t]{0.48\linewidth}
\centering
\begin{tabular}[t]{@{}lccc@{}}
\multicolumn{4}{c}{\textbf{ZINC}}\\
\toprule
Model & Val.$\uparrow$ & Unique.$\uparrow$ & FCD$\downarrow$ \\
\midrule
GruM & 98.7 & -- & 2.26 \\
GBD & 97.9 & -- & 2.25 \\
CatFlow & \underline{99.2} & \textbf{100.0} & 13.21 \\
GGFlow & \textbf{99.6} & \textbf{100.0} & 1.45 \\
DeFoG (retrained $L=8$) & \underline{99.2} & \textbf{100.0} & 1.43 \\
DeFoG (retrained $L=16$) & 95.3 & \textbf{100.0} & 1.81 \\
\midrule
\rowcolor{gray!15}
\model{} ($L=8$) & 98.4 & \textbf{100.0} & 0.94 \\
\rowcolor{gray!15}
\model{} ($L=16$) & 98.1 & \textbf{100.0} & \underline{0.86} \\
\rowcolor{gray!15}
\model{} ($L=32$) & 98.1 & \textbf{100.0} & \textbf{0.85} \\
\bottomrule
\end{tabular}
\end{minipage}\hfill
\begin{minipage}[t]{0.48\linewidth}
\centering
\begin{tabular}[t]{@{}lccc@{}}
\multicolumn{4}{c}{\textbf{MOSES}}\\
\toprule
Model & Val.$\uparrow$ & Unique.$\uparrow$ & FCD$\downarrow$ \\
\midrule
Training set & 100.0 & 100.0 & 0.64 \\
\midrule
GraphInvent & \textbf{96.4} & 99.8 & 1.22 \\
DiGress & 85.7 & \textbf{100.0} & 1.19 \\
DisCo & 88.3 & \textbf{100.0} & 1.44 \\
Cometh & 90.5 & \underline{99.9} & 1.27 \\
SimGFM & 89.4 & \textbf{100.0} & \textbf{1.08} \\
DeFoG (retrained $L=8$) & 89.3 & \underline{99.9} & 1.34 \\
DeFoG (retrained $L=16$) & 86.1 & \underline{99.9} & 1.68 \\
\midrule
\rowcolor{gray!15}
\model{} ($L=8$) & 91.7 & \textbf{100.0} & 1.23 \\
\rowcolor{gray!15}
\model{} ($L=16$) & \underline{92.7} & \textbf{100.0} & 1.15 \\
\rowcolor{gray!15}
\model{} ($L=32$) & 92.1 & \textbf{100.0} & \underline{1.10} \\
\bottomrule
\end{tabular}
\end{minipage}
\end{table}
 We next test whether the advantages of depth carry over to molecular
generation, where models must jointly ensure chemical validity and match the target distribution.
Table~\ref{tab:molecular_generation} reports results from $10{,}000$ (ZINC) and $25{,}000$ (MOSES)
molecules generated by \model{} with $L\in\{8,16,32\}$ on ZINC and
MOSES.

For these experiments, we set $\gamma_0=0$, corresponding to the fully non-dissipative regime. This choice is motivated by the results in \Cref{sec:synthetic_experiments} and by the ablation in Appendix~\ref{app:ablation_gamma}, where increasing dissipation consistently degrades generation quality. We therefore use the non-dissipative variant of \model{} for molecular generation, while standard contractive GT-based models, such as DiGress and DeFoG, provide the comparison with dissipative dynamics. Additional metrics for both ZINC and MOSES, including training cost, are reported in Appendix~\ref{app:additinal_metrics}.

\textbf{ZINC.}
Generation quality improves with depth: FCD decreases from $0.94$ at $L=8$ to $0.85$ at $L=32$, with a further improvement at $L=32$. At matched depth, \model{} consistently outperforms retrained DeFoG, whose FCD instead worsens from $1.43$ at $L=8$ to $1.81$ at $L=16$. DeFoG's validity also declines from $99.2\%$ to $95.3\%$, whereas \model{} achieves $98.1\%$ at both $L=16$ and $L=32$. Uniqueness remains at $100.0\%$ for both methods.
Hence, added depth improves \model{}'s distributional fit while preserving high validity; deeper DeFoG models instead degrade on both metrics.

\textbf{MOSES.}
A depth-dependent trend similar to ZINC emerges on MOSES.
At $L=8$, \model{} already improves over retrained DeFoG in both validity ($91.7\%$ versus $89.3\%$) and FCD ($1.23$ versus $1.34$).
At $L=16$, validity rises to $92.7\%$; with $L=32$ our \model{} further raises validity to $92.1\%$, while dropping FCD at 1.10.
Uniqueness remains at $100.0\%$ across all depths. Conversely, deepening DeFoG from $L=8$ to $L=16$ reduces validity from $89.3\%$ to $86.1\%$ and worsens FCD from $1.34$ to $1.68$.
Therefore, additional layers progressively improve \model{} while degrading the standard denoiser.
Overall, these results indicate that orthogonal node transport is not confined to very deep architectures and it enables deeper denoisers to improve distributional agreement on both datasets and validity on MOSES.

\subsection{Relating Depth Gains to Representation Dynamics}
\label{sec:experiments-dynamics}

To connect these gains to the mechanism in Section~\ref{sec:theory}, we probe the forward dynamics of trained Planar denoisers across increasing depths. Following~\citep{attention_is_not_all_you_need}, we track Dong residual
$\|\mathbf{X}-\mathbf{1}\bar{\mathbf{x}}^{\top}\|_F/\|\mathbf{X}\|_F$, where
$\bar{\mathbf{x}}$ denotes the mean node representation, and the effective rank 
of the node-feature matrix~\citep{effectiverank}, $\exp[-\sum_i p_i\log p_i]$ with
$p_i=\sigma_i(\mathbf{X})/\sum_j\sigma_j(\mathbf{X})$. Dong residual near zero or effective rank near one indicates collapse
to a shared node representation. 
\Cref{fig:method_rank_collapse} shows that the baseline denoiser progressively loses representational diversity with depth, as both Dong residual and effective rank approach their collapse limits. In contrast, \model{}, trained at the same depth, preserves both metrics, maintaining distinct node representations throughout the forward pass. This difference is most pronounced in the deep regime, where the baseline collapses and \model{} achieves its best generation quality according to \Cref{tab:synthetic_main}. These results align the empirical gains from depth with the predicted dynamical behavior: while standard transport progressively erases node-specific information, the proposed non-dissipative transport preserves it across layers, enabling \model{} to exploit additional depth without representation collapse.

\begin{figure}[t]
    \centering
    \includegraphics[width=\linewidth]{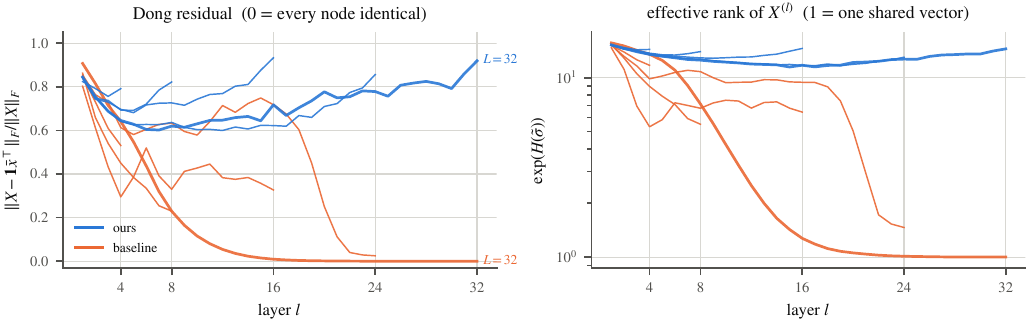}
    \caption{Representation dynamics of trained Planar models across
    depth. \textbf{Left:} Dong residual between node states. \textbf{Right:} effective rank of the node-feature matrix. Each curve ends at its
    model depth; the $L=32$ curves are highlighted. Orthogonal transport
    mitigates the collapse observed in the deep baseline for these checkpoints.}
    \label{fig:method_rank_collapse}
    \vspace{-3mm}
\end{figure}

 \section{Conclusion}\label{sec:conclusion}

We investigated graph generation through the dynamical regime induced by Graph Transformer denoisers. Our analysis reveals that standard self-attention becomes increasingly contractive with depth, causing gradients to vanish and node-specific information to collapse. This finding motivated \model{}, a permutation-equivariant GT with intrinsically stable, non-dissipative transport. Across synthetic and molecular benchmarks, the advantage of its non-dissipative dynamics grows with depth: deeper contractive denoisers degrade and eventually collapse, whereas \model{} continues to benefit from added layers; within \model{}, increasing damping progressively degrades generation quality. These gains align with sustained gradient flow and preserved diversity among node representations in deep regimes. Taken together, our results establish spectral dynamics, not architectural capacity alone, as a key determinant of whether depth strengthens or undermines graph generative models.

\bibliographystyle{iclr2027_conference}

\appendix

\section{Proofs for the \modelext{}}
\label{app:orthogonal_transformer}

We use the notation of Section~\ref{sec:method}. All orthogonality and
conditioning statements concern active node coordinates and exact arithmetic.
Unless stated otherwise, layer parameters are fixed. The conditional transport
Jacobian holds the edge and global inputs fixed and uses the detached-score
differential specified in Theorem~\ref{thm:orthogonal_transport}.

\subsection{Proof of Theorem~\ref{thm:orthogonal_transport}} 
\orthogonaltransport*

\begin{proof}
\label{app:transport_proof}
Consider a single layer $\ell$ and a single attention head $h$.
For each $h$, skew-symmetry of $\mathbf{K}^{(h)}$
implies that its Cayley transform $\mathbf{R}^{(h)}$ is
orthogonal~\citep{helfrich2018orthogonalrecurrentneuralnetworks}.
The orthogonal head update is then:
\[
    \mathbf{Z}^{(h)}
    = \mathbf{R}^{(h)}\mathbf{X}^{(h)}\mathbf{C}^{(h)}.
\]
Holding the attention matrix fixed, its element-wise derivative is
\[
    \frac{\partial Z^{(h)}_{ia}}{\partial X^{(h)}_{jb}}
    = R^{(h)}_{ij}C^{(h)}_{ba}.
\]
With column-wise vectorization the Jacobian of a single head is therefore
\[
    \mathbf{J}^{(h)}
    = \mathbf{C}^{(h)\top}\otimes\mathbf{R}^{(h)}.
\]
Using the transpose and multiplication identities for
Kronecker products, we obtain
\begin{align}
    \mathbf{J}^{(h)\top}\mathbf{J}^{(h)}
    &=
    \bigl(\mathbf{C}^{(h)}\otimes\mathbf{R}^{(h)\top}\bigr)
    \bigl(\mathbf{C}^{(h)\top}\otimes\mathbf{R}^{(h)}\bigr)
    \nonumber\\
    &=
    \bigl(\mathbf{C}^{(h)}\mathbf{C}^{(h)\top}\bigr)
    \otimes
    \bigl(\mathbf{R}^{(h)\top}\mathbf{R}^{(h)}\bigr)
    = \mathbf{I}_{Nd_h}.
\end{align}

Since the heads act on independent and disjoint blocks, their
concatenation has the orthogonal Jacobian
\[
    \mathbf{D}
    = \operatorname{diag}
      \bigl(\mathbf{J}^{(1)},\ldots,\mathbf{J}^{(H)}\bigr).
\]
Right multiplication by $\mathbf{O}$ has Jacobian
$\mathbf{O}^{\top}\otimes\mathbf{I}_N$, while the additive
bias has zero derivative under the fixed-conditioning
assumption. Thus,
\[
    \mathbf{J}_{\ell}
    = (\mathbf{O}^{\top}\otimes\mathbf{I}_N)\mathbf{D},
\]
and
\[
    \mathbf{J}_{\ell}^{\top}\mathbf{J}_{\ell}
    = \mathbf{D}^{\top}
      \bigl((\mathbf{O}\mathbf{O}^{\top})
      \otimes\mathbf{I}_N\bigr)\mathbf{D}
    = \mathbf{I}_{Nd}.
\]

This identity implies that every singular value of
$\mathbf{J}_{\ell}$ equals one.
For any eigenpair
$\mathbf{J}_{\ell}\mathbf{v}=\lambda\mathbf{v}$,
with $\mathbf{v}\in\mathbb{C}^{Nd}\setminus\{0\}$,
orthogonality also gives
\[
    \|\mathbf{v}\|_2
    = \|\mathbf{J}_{\ell}\mathbf{v}\|_2
    = |\lambda|\,\|\mathbf{v}\|_2,
\]
hence $|\lambda|=1$.
Finally, the chain rule expresses the propagation Jacobian
across $L$ layers as
$\mathbf{J}_{L-1}\cdots\mathbf{J}_0$.
Since a product of orthogonal matrices is orthogonal, this completes the proof.
\end{proof}

\subsection{Damped Spectrum and Composition Bounds}
\label{app:damping_proof}

\controlleddissipation*

\begin{proof}[Proof of Proposition~\ref{prop:controlled_dissipation}]
Let $\gamma\geq 0$. Since $\mathbf{K}$ is real and
skew-symmetric, its eigenvalues are purely imaginary.
Let $\mathbf{v}\in\mathbb{C}^{N}\setminus\{0\}$ satisfy
$\mathbf{K}\mathbf{v}=\mathrm{i}\omega\mathbf{v}$,
with $\omega\in\mathbb{R}$. By definition,
\begin{equation}
    \mathbf{R}_{\gamma}
    =
    \left((1+\gamma/2)\mathbf{I}-\tfrac12\mathbf{K}\right)^{-1}
    \left((1-\gamma/2)\mathbf{I}+\tfrac12\mathbf{K}\right).
\end{equation}
The matrix being inverted is nonsingular, since each of its
eigenvalues has real part $1+\gamma/2>0$.

Applying the second factor to $\mathbf{v}$ gives
\begin{align}
    \left((1-\gamma/2)\mathbf{I}+\tfrac12\mathbf{K}\right)\mathbf{v}
    &=(1-\gamma/2)\mathbf{v}+\tfrac12\mathbf{K}\mathbf{v}
    \nonumber\\
    &=\left(1-\gamma/2+\mathrm{i}\omega/2\right)\mathbf{v}.
\end{align}
Similarly, $\mathbf{v}$ is an eigenvector of the inverting factor,
\begin{equation}
    \left((1+\gamma/2)\mathbf{I}-\tfrac12\mathbf{K}\right)\mathbf{v}
    =\left(1+\gamma/2-\mathrm{i}\omega/2\right)\mathbf{v},
\end{equation}
so it is also an eigenvector of its inverse, with the reciprocal
eigenvalue $\bigl(1+\gamma/2-\mathrm{i}\omega/2\bigr)^{-1}$
(the scalar is nonzero because its real part is $1+\gamma/2>0$):
\begin{equation}
    \left((1+\gamma/2)\mathbf{I}-\tfrac12\mathbf{K}\right)^{-1}
    \mathbf{v}
    =
    \frac{1}{1+\gamma/2-\mathrm{i}\omega/2}\,\mathbf{v}.
\end{equation}
Combining these identities, we obtain
\begin{equation}
    \mathbf{R}_{\gamma}\mathbf{v}
    =
    \frac{1-\gamma/2+\mathrm{i}\omega/2}
         {1+\gamma/2-\mathrm{i}\omega/2}\,\mathbf{v}.
\end{equation}
Thus the eigenvalue corresponding to $\mathbf{v}$ is
$r_{\gamma}(\omega)$ as stated in
\eqref{eq:damped_spectrum}. Taking its squared modulus gives
\begin{equation}
    |r_{\gamma}(\omega)|^2
    =
    \frac{(1-\gamma/2)^2+\omega^2/4}
         {(1+\gamma/2)^2+\omega^2/4}.
\end{equation}

It remains to show that $\mathbf{R}_{\gamma}$ is normal with
singular values $|r_{\gamma}(\omega)|$. Being real and
skew-symmetric, $\mathbf{K}$ is normal, hence unitarily
diagonalizable: $\mathbf{K}=\mathbf{U}\boldsymbol{\Lambda}\mathbf{U}^{*}$
with $\mathbf{U}$ unitary and
$\boldsymbol{\Lambda}=\operatorname{diag}(\mathrm{i}\omega_1,\dots,\mathrm{i}\omega_N)$.
Since both factors defining $\mathbf{R}_{\gamma}$ are polynomials in
$\mathbf{K}$, they are simultaneously diagonalized by the same
$\mathbf{U}$, and therefore
\begin{equation}
    \mathbf{R}_{\gamma}
    =
    \mathbf{U}\,r_{\gamma}(\boldsymbol{\Lambda})\,\mathbf{U}^{*},
    \qquad
    r_{\gamma}(\boldsymbol{\Lambda})
    =
    \operatorname{diag}\bigl(r_{\gamma}(\omega_1),\dots,r_{\gamma}(\omega_N)\bigr).
\end{equation}
Thus $\mathbf{R}_{\gamma}$ is unitarily diagonalizable and hence
normal. For a normal matrix the singular values coincide with the
moduli of its eigenvalues, so the singular values of
$\mathbf{R}_{\gamma}$ are exactly $|r_{\gamma}(\omega)|$.

In particular,
\begin{equation}
    1-|r_{\gamma}(\omega)|^2
    =
    \frac{2\gamma}{(1+\gamma/2)^2+\omega^2/4}.
\end{equation}
Hence $|r_{\gamma}(\omega)|=1$ when $\gamma=0$,
whereas $|r_{\gamma}(\omega)|<1$ when $\gamma>0$.
\end{proof}

\begin{corollary}[Depth-scaled attenuation]
\label{cor:depth_scaled_damping}
Consider $L$ damped transport layers with orthogonal channel matrices
$\mathbf{C}^{(\ell,h)}$ and orthogonal (or identity) cross-head mixing
$\mathbf{O}^{(\ell)}$, holding attention scores and conditioning fixed when
differentiating. For $0\leq\gamma<2$, define
\[
    a_{\gamma}
    := \frac{1-\gamma/2}{1+\gamma/2}.
\]
Each conditional node-propagation Jacobian $\mathbf{J}_{\ell}$
has singular values in $[a_{\gamma},1]$, and their composition
satisfies
\begin{equation}
    a_{\gamma}^{L}
    \leq \sigma_{\min}(\mathbf{J}_{L-1}\cdots\mathbf{J}_0)
    \leq \sigma_{\max}(\mathbf{J}_{L-1}\cdots\mathbf{J}_0)
    \leq 1.
    \label{eq:depth_scaled_damping}
\end{equation}
In particular, for fixed $\gamma_0\geq 0$ and
$\gamma=\gamma_0/L$ with $L>\gamma_0/2$,
\begin{equation}
    \lim_{L\to\infty}a_{\gamma_0/L}^{L}
    = e^{-\gamma_0}.
    \label{eq:damping_limit}
\end{equation}
\end{corollary}

\begin{proof}
Fix a layer $\ell$ and omit its index. As in the proof of
Theorem~\ref{thm:orthogonal_transport}, with scores and conditioning held
fixed, the damped attention submap has Jacobian
\[
    \mathbf{J}
    = (\mathbf{O}^{\top}\otimes\mathbf{I}_N)\,
      \operatorname{diag}\bigl(
      \mathbf{C}^{(1)\top}\otimes\mathbf{R}^{(1)}_{\gamma},\ldots,
      \mathbf{C}^{(H)\top}\otimes\mathbf{R}^{(H)}_{\gamma}\bigr).
\]
The singular values of a Kronecker product are the pairwise products of the
singular values of its factors. Since $\mathbf{C}^{(h)}$ is orthogonal, the
singular values of $\mathbf{C}^{(h)\top}\otimes\mathbf{R}^{(h)}_{\gamma}$ are
those of $\mathbf{R}^{(h)}_{\gamma}$, each repeated $d_h$ times. A
block-diagonal matrix has the union of the singular values of its blocks, and
left multiplication by the orthogonal matrix
$\mathbf{O}^{\top}\otimes\mathbf{I}_N$ leaves singular values unchanged.
By Proposition~\ref{prop:controlled_dissipation}, the singular values of
$\mathbf{J}$ are therefore the moduli $|r_{\gamma}(\omega)|$, where
$\mathrm{i}\omega$ ranges over the eigenvalues of the generators
$\mathbf{K}^{(h)}$.

Write $s=\omega^2/4\geq 0$, $\alpha=(1-\gamma/2)^2$ and
$\beta=(1+\gamma/2)^2$, so that
$|r_{\gamma}(\omega)|^2=(\alpha+s)/(\beta+s)$. Since $\alpha\leq\beta$, this
function is nondecreasing in $s$ and bounded above by one; its minimum over
$s\geq 0$ is $\alpha/\beta$, attained at $s=0$. For $0\leq\gamma<2$ we have
$\sqrt{\alpha/\beta}=a_{\gamma}$, hence every singular value of
$\mathbf{J}_{\ell}$ lies in $[a_{\gamma},1]$.

For square matrices $\mathbf{A},\mathbf{B}$,
$\sigma_{\min}(\mathbf{A}\mathbf{B})\geq\sigma_{\min}(\mathbf{A})\,\sigma_{\min}(\mathbf{B})$
and
$\sigma_{\max}(\mathbf{A}\mathbf{B})\leq\sigma_{\max}(\mathbf{A})\,\sigma_{\max}(\mathbf{B})$.
Applying these inequalities inductively to $\mathbf{J}_{L-1}\cdots\mathbf{J}_0$
yields \Cref{eq:depth_scaled_damping}.

Finally, let $x=\gamma_0/(2L)\in[0,1)$. Then
\[
    L\log a_{\gamma_0/L}
    = L\bigl[\log(1-x)-\log(1+x)\bigr]
    = -2Lx + O(Lx^{3})
    = -\gamma_0 + O\!\left(\gamma_0^{3}/L^{2}\right),
\]
which tends to $-\gamma_0$ as $L\to\infty$, proving \Cref{eq:damping_limit}.
\end{proof}

\paragraph{Padding and damping range.}
After grouping active and padded coordinates, the masked generator and damping
matrix take the block forms
\[
    \mathbf{K}_{\mathrm{pad}}
    =\begin{pmatrix}\mathbf{K}_{\mathrm{act}}&0\\0&0\end{pmatrix},
    \mathbf{G}
    =\begin{pmatrix}\gamma\mathbf{I}_N&0\\0&0\end{pmatrix}.
\]
Their Cayley map is
$\operatorname{diag}(\operatorname{cay}(\mathbf{K}_{\mathrm{act}}
-\gamma\mathbf{I}_N),\mathbf{I})$. This establishes decoupling and the
identity action on padding. Normality here follows from the uniform damping
on the active block; arbitrary unequal nodewise damping need not commute with
the generator. At $\gamma=2$, zero-frequency modes are annihilated, so the
positive lower bound requires $\gamma<2$. Contractivity itself holds for
every $\gamma>0$.

\subsection{Permutation Equivariance}
\label{app:permutation_equivariance}

Let $\mathbf{P}\in\mathbb{R}^{M\times M}$ be a permutation matrix,
where $M$ includes padded coordinates. The node mask
$\mathbf{m}\in\{0,1\}^{M}$ satisfies $m_i=1$ for active nodes
and $m_i=0$ for padding, so that
$N=\mathbf{1}^{\top}\mathbf{m}$ is the active-node count.
For each edge channel $c$, define
$(\mathbf{P}\cdot\mathbf{E})_{::c}
=\mathbf{P}\mathbf{E}_{::c}\mathbf{P}^{\top}$.
The action on the graph state is
\begin{equation}
    \mathbf{P}\cdot(\mathbf{X},\mathbf{E},\mathbf{y},\mathbf{m})
    =(\mathbf{P}\mathbf{X},\mathbf{P}\cdot\mathbf{E},
      \mathbf{y},\mathbf{P}\mathbf{m}).
    \label{eq:app_permutation_action}
\end{equation}

We next show that the Cayley-based construction preserves this permutation action, yielding permutation equivariance of \model{}.
\begin{theorem}[Permutation equivariance]
    \model{} is permutation equivariant.
\end{theorem}

\begin{proof}
Assume that nodewise and pairwise maps share parameters across
nodes and ordered node pairs, respectively, LayerNorm acts on
feature channels, and global readouts are permutation invariant.
We first consider dropout disabled.

Fix a layer and omit its index $\ell$. Following the main text,
$\mathbf{S}_h$ denotes the edge-modulated attention-score matrix
for head $h$, and $\mathbf{K}_h$ its skew-symmetric node
generator. Writing
$\mathbf{D}_{\mathbf{m}}=\operatorname{diag}(\mathbf{m})$,
the padded version of \Cref{eq:node_generator} is
\begin{equation}
    \mathbf{K}_h
    =\frac{\varepsilon_h(\mathbf{y})}{2\sqrt{N}}\,
      \mathbf{D}_{\mathbf{m}}
      (\mathbf{S}_h-\mathbf{S}_h^\top)
      \mathbf{D}_{\mathbf{m}}.
\end{equation}
For convenience, denote the damped generator by
$\mathbf{A}_h:=\mathbf{K}_h-\gamma\mathbf{D}_{\mathbf{m}}$,
where $\gamma=\gamma_0/L$. Its Cayley transform
$\mathbf{R}_{\gamma,h}=\cayley(\mathbf{A}_h)$
is the node-transport matrix. The orthogonal construction
corresponds to $\gamma=0$.

\paragraph{Scores and generators.}
Primes denote quantities computed from the permuted input.
Shared query--key maps and edge modulation apply the same
score function to every ordered node pair. Relabeling the
nodes therefore reorders both score indices:
\begin{equation}
    \mathbf{S}_h'
    =\mathbf{P}\mathbf{S}_h\mathbf{P}^{\top}.
\end{equation}
This identity ensures that the transport matrix constructed
from these scores also transforms consistently under
relabeling\cite{vignac2023digress}. Indeed, the global gate is unchanged, and
\[
    N'=\mathbf{1}^{\top}\mathbf{P}\mathbf{m}=N,
    \qquad
    \mathbf{D}_{\mathbf{P}\mathbf{m}}
    =\mathbf{P}\mathbf{D}_{\mathbf{m}}\mathbf{P}^{\top}.
\]
Using $\mathbf{P}^{\top}\mathbf{P}=\mathbf{I}$ gives
\begin{align}
    \mathbf{K}_h'
    &=
    \frac{\varepsilon_h(\mathbf{y})}{2\sqrt{N}}\,
    \mathbf{P}\mathbf{D}_{\mathbf{m}}
    (\mathbf{S}_h-\mathbf{S}_h^\top)
    \mathbf{D}_{\mathbf{m}}\mathbf{P}^{\top}
    =\mathbf{P}\mathbf{K}_h\mathbf{P}^{\top},
    \\
    \mathbf{A}_h'
    &=\mathbf{K}_h'-\gamma\mathbf{D}_{\mathbf{P}\mathbf{m}}
    =\mathbf{P}\mathbf{A}_h\mathbf{P}^{\top}.
\end{align}
Detachment changes no forward values and preserves these
identities.

\paragraph{Node transport.}
The Cayley transform commutes with permutation conjugation:
\begin{align}
    \mathbf{R}_{\gamma,h}'
    &=
    \left[\mathbf{P}(\mathbf{I}-\mathbf{A}_h/2)
    \mathbf{P}^{\top}\right]^{-1}
    \mathbf{P}(\mathbf{I}+\mathbf{A}_h/2)\mathbf{P}^{\top}
    \nonumber\\
    &=\mathbf{P}\mathbf{R}_{\gamma,h}\mathbf{P}^{\top}.
\end{align}
Recall that $\mathbf{X}_h$ contains the node features of
head $h$, and $\mathbf{C}_h$ mixes its feature channels.
Since $\mathbf{X}_h'=\mathbf{P}\mathbf{X}_h$ and
$\mathbf{C}_h$ is independent of node labels,
\begin{align}
    \mathbf{Z}_h'
    &=\mathbf{R}_{\gamma,h}'\mathbf{X}_h'\mathbf{C}_h
    \nonumber\\
    &=(\mathbf{P}\mathbf{R}_{\gamma,h}\mathbf{P}^{\top})
      (\mathbf{P}\mathbf{X}_h)\mathbf{C}_h
    =\mathbf{P}\mathbf{Z}_h.
\end{align}
Thus permuting the input permutes each transported head
in exactly the same way. Head concatenation and the
channel-mixing matrix $\mathbf{O}$ preserve this identity.
The broadcast bias does also, since
$\mathbf{P}\mathbf{1}=\mathbf{1}$:
\begin{align}
    \widetilde{\mathbf{X}}'
    &=
    \operatorname{Concat}_h(\mathbf{P}\mathbf{Z}_h)\mathbf{O}
    +\mathbf{1}\mathbf{b}(\mathbf{y})^\top
    \nonumber\\
    &=\mathbf{P}\left[
      \operatorname{Concat}_h(\mathbf{Z}_h)\mathbf{O}
      +\mathbf{1}\mathbf{b}(\mathbf{y})^\top\right]
    =\mathbf{P}\widetilde{\mathbf{X}}.
\end{align}
Applying the permuted node mask preserves this equality.

\paragraph{Complete denoiser.}
The edge branch consists of shared pairwise operations,
so its output transforms by the same permutation of both
node indices. Invariant node and edge readouts leave the
updated global representation unchanged. Shared feed-forward
maps, channelwise normalization, residual additions, and
layer scalars preserve these transformation laws.
Each complete layer is therefore equivariant, and induction
extends the result to the entire stack \cite{vignac2023digress}.

Input/output maps, masks, and input-to-output residuals
obey the same laws. Time embeddings are global; edge
symmetrization commutes with permutation conjugation,
as does diagonal removal because
$\mathbf{P}\mathbf{I}\mathbf{P}^{\top}=\mathbf{I}$.
Consequently, for every $\gamma_0\geq0$,
\begin{equation}
    F(\mathbf{P}\cdot\mathbf{s})
    =\mathbf{P}\cdot F(\mathbf{s}),
    \qquad
    \mathbf{s}=(\mathbf{X},\mathbf{E},\mathbf{y},\mathbf{m}).
    \label{eq:app_denoiser_equivariance}
\end{equation}

Finally, with dropout enabled, relabeling its masks
$\boldsymbol{\xi}$ along with the graph gives
\[
    F_{\mathbf{P}\cdot\boldsymbol{\xi}}
      (\mathbf{P}\cdot\mathbf{s})
    =\mathbf{P}\cdot F_{\boldsymbol{\xi}}(\mathbf{s}).
\]
Independent identically distributed dropout masks have
the same law after relabeling, establishing equivariance
in distribution.
\end{proof}

\section{Additional Experimental Details and Results}
\label{app:additional_experiments}
This section provides additional details on the experimental setup and complements the results reported in the main text. For all experiments, we build on the public implementation of DeFoG~\citep{qin2025defog} and reuse its released training and sampling hyperparameters for each dataset (optimizer, learning rate, batch size, number of training epochs, and number of sampling steps), for both DeFoG and \model{} and at every depth $L$. The only differences between the two models are
the architectural changes introduced by \model{} in \Cref{sec:method}.
Furthermore, we first describe the global conditioning mechanism used by SGT, followed by details on the employed baselines.
We then report additional molecular generation metrics and an ablation on the damping parameter $\gamma_0$ to further assess the effect of increasingly dissipative dynamics.

\paragraph{Global Conditioning} Following DeFoG, each layer receives a graph-level vector $y$ recomputed from the noisy graph $G_t$ at every step. It concatenates the normalized node count $n/n_{\max}$, the counts of 3-, 4-, 5- and 6-cycles in $G_t$ (each scaled by $1/10$ and clipped to $[0,1]$), the normalized molecular weight of $G_t$ on the molecular datasets, and the flow time $t\in[0,1]$. No class or property label is used: all runs are unconditional. A 64-dimensional sinusoidal embedding of $t$ is appended, and a two-layer MLP maps the result to $d_y$. In the orthogonal block $y$ enters the node update in exactly two places. It sets the per-head rotation angle, $K^h=\varepsilon^h(y)\,(S^h-S^{h\top})/(2\sqrt{n})$ with $\varepsilon^h(y)$ a learned gate,
and contributing an additive bias in$X\leftarrow R\,X\,C + W_x y$. Because the bias is a constant shift, it leaves the node Jacobian unchanged. The time dependence therefore changes how far each layer rotates, never whether the map is orthogonal. This replaces the multiplicative FiLM gate on $X$ used by the baseline, which could rescale node states. The edge stream keeps the baseline's FiLM modulation $E\leftarrow W^{a}_e y + (1+W^{m}_e y)\odot E$. The $y$ stream is updated residually from pooled node and edge statistics.

\subsection{Employed baselines}\label{app:baselines}
In our experiments, the performance of our method is compared with various state-of-the-art generative models from the literature. Specifically, we consider:
\begin{itemize}
    \item \emph{Autoregressive and recurrent models}, including
    GraphRNN \citep{you2018graphrnngeneratingrealisticgraphs},
    GRAN \citep{gran},
    BiGG \citep{bigg},
    GraphGen \citep{graphgen},
    and GraphInvent \citep{mercado2021graphinvent}.
    
    \item \emph{Spectral and hierarchical models}, including
    SPECTRE \citep{sbmplanar}
    and HSpectre \citep{hspectre}.
    
    \item \emph{GNN-based diffusion models}, including
    EDGE \citep{edge},
    EDP-GNN \citep{edpgnn},
    BwR (EDP-GNN) \cite{bwr}

    \item \emph{Graph Transformer-based diffusion and flow models}, including
    DiGress \citep{vignac2023digress},
    DisCo \citep{xu2024discretestatecontinuoustimediffusiongraph},
    Cometh \citep{siraudin2024comethcontinuoustimediscretestategraph},
    GBD \citep{gbd},
    CatFlow \citep{eijkelboom2025variationalflowmatchinggraph},
    GGFlow \citep{hou2024improvingmoleculargraphgeneration},
    SimGFM \citep{luo2026simgfm},
    GruM \citep{jo2024graphgenerationdiffusionmixture}.
    and DeFoG \citep{qin2025defog}.
\end{itemize}

\subsection{Additional Metrics}\label{app:additinal_metrics}
We complement the main evaluation with additional molecular metrics
to examine how depth affects distributional agreement, structural
similarity, and diversity.
The extended MOSES evaluation uses $25{,}000$ generated molecules
and $500$ sampling steps, with results computed against both the
standard test set (Test) and the scaffold-split test set (TestSF).
The ZINC250k evaluation uses $10{,}000$ generated molecules and
a single evaluation fold.
Training-data scores provide an empirical reference at the same
sample size.

\paragraph{MOSES.}
Increasing the depth of \model{} from $L=8$ to $L=16$ and $L=32$ improves
all reported metrics.
FCD decreases from $0.793$ to $0.705$ on Test and from $1.1.51$
to $1.101$ on TestSF, with the deepest configuraiton.
These improvements are accompanied by higher nearest-neighbor
similarity (SNN), which increases from $0.587$ to $0.600$ on
Test and from $0.556$ to $0.565$ on TestSF.
Scaffold similarity also increases on both splits, while the
fraction of molecules passing the filters rises from $99.02\%$
to $99.21\%$.
Thus, the improvement in FCD is accompanied by gains in the
additional structural metrics.

Conversely, DeFoG exhibits a different response to depth.
Its FCD worsens from $0.818$ to $1.129$ on Test and from $1.339$
to $1.676$ on TestSF, despite modest improvements in SNN.
Its scaffold similarity increases on Test but decreases on TestSF.
At $L=16$, \model{} achieves lower FCD and higher SNN than DeFoG
on both splits, with the same filter pass rate.
Interestingly, for our model, TestSF scaffold
similarity is nearly identical at this depth ($11.09\%$ versus
$11.10\%$), and DeFoG at $L=8$ achieves the highest value
($12.13\%$).

\paragraph{ZINC250k.}
Increasing the depth of \model{} from $L=8$ to $L=32$ likewise
improves all reported metrics.
FCD decreases from $0.943$ to $0.851$, SNN increases from $0.434$
to $0.452$, and scaffold similarity rises from $60.38\%$ to
$66.37\%$ with 16 layers.
Fragment similarity and internal diversity also increase slightly,
from $0.996$ to $0.998$ and from $0.863$ to $0.864$, respectively.

\begin{table}[ht]
\caption{MOSES benchmark metrics (molsets), test phase on $25{,}000$ generated molecules, 500 sampling steps. FCD, SNN and Scaf are measured against the MOSES test set (Test) and the scaffold test set (TestSF); Filters and Scaf are percentages. \emph{Train data} scores $n$ random training molecules and is the reference for a perfect model at the same $n$ (Scaf/TestSF is $0$ by construction: the scaffold split shares no scaffolds with training). Best model value in bold, second best underlined.}
\label{tab:moses_full}
\centering\small
\begin{tabular}{llccccccc}
\toprule
 & & & \multicolumn{3}{c}{Test} & \multicolumn{3}{c}{TestSF} \\
\cmidrule(lr){4-6}\cmidrule(lr){7-9}
Model & Depth & Filters $\uparrow$ & FCD $\downarrow$ & SNN $\uparrow$ & Scaf $\uparrow$ & FCD $\downarrow$ & SNN $\uparrow$ & Scaf $\uparrow$ \\
\midrule
DeFoG & $L=8$ & \textbf{99.21} & 0.818 & 0.584 & 85.28 & 1.339 & 0.552 & \textbf{12.13} \\
 & $L=16$ & \textbf{99.21} & 1.129 & 0.591 & 86.52 & 1.676 & 0.556 & 11.10 \\
\addlinespace
\model{} & $L=8$ & \underline{99.02} & 0.793 & 0.587 & 85.38 & 1.229 & 0.556 & 10.26 \\
& $L=16$ & \textbf{99.21} & \underline{0.777} & \textbf{0.600} & \underline{86.67} & \underline{1.151} & \textbf{0.567} & 11.09 \\
& $L=32$ & \textbf{99.21} & \textbf{0.705} & \underline{0.597} & \textbf{90.63} & \textbf{1.101} & \underline{0.565} & \underline{11.45} \\
\midrule
\multicolumn{2}{l}{\emph{Train data}} & 100.00 & 0.134 & 0.643 & 88.84 & 0.636 & 0.586 & 0.00 \\
\bottomrule
\end{tabular}
\end{table} \begin{table}[ht]
\caption{ZINC250k, MOSES benchmark metrics (molsets) against the ZINC250k test set, test phase on $n=10^4$ generated molecules, single fold. Scaf is a percentage. \emph{Train data} scores $10^4$ random training molecules, the reference for a perfect model at the same $n$; Scaf and IntDiv can exceed it. ZINC250k has no scaffold split, so only the Test variants exist. Best model value in bold, second best underlined.}
\label{tab:zinc_molsets}
\centering\small
\begin{tabular}{llccccc}
\toprule
Model & Depth & FCD $\downarrow$ & SNN $\uparrow$ & Scaf $\uparrow$ & Frag $\uparrow$ & IntDiv $\uparrow$ \\
\midrule
DeFoG & $L=8$ & 1.43 & 0.433 & 62.17 & 0.995 & 0.862 \\
 & $L=16$ & 1.30 & 0.441 & 61.50 & 0.990 & 0.863 \\
\addlinespace
\model{} & $L=8$ & 0.94& 0.434 & 60.38 & \underline{0.996} & 0.863 \\
 & $L=16$ & \underline{0.86}& \underline{0.442} & \textbf{66.37} & \textbf{0.998} & \underline{0.864} \\
 & $L=32$ & \textbf{0.85}& \textbf{0.452} & \underline{66.13} & \textbf{0.998} & \textbf{0.868} \\
\midrule
\multicolumn{2}{l}{\emph{Train data}} & 0.21& 0.483 & 63.22 & 1.000 & 0.869 \\
\bottomrule
\end{tabular}

\end{table} 
\paragraph{Training cost.}
\Cref{tab:runtime-simple} compares model size and training time at matched depths of our \model{} against the standard GT used in \cite{qin2025defog}, measured on a single GPU of the platform described in \Cref{tab:system-config}.
Across all available comparisons, \model{} uses $11.4$--$15.8\%$ fewer
parameters than the baseline while maintaining comparable or lower
per-epoch training times, with reductions reaching approximately $13\%$
on SBM and $10\%$ on ZINC250k.
On MOSES, the baseline at $L=32$ failed to train effectively due to
vanishing gradients, whereas \model{} remained trainable.
These results complement the generation experiments: the proposed
architecture enables effective training at greater depth without
increasing the measured per-epoch cost at matched depths.

\begin{table}[ht]
\centering
\caption{Hardware and software configuration of the experimental platform.}
\label{tab:system-config}
\small
\begin{tabular}{@{}ll@{}}
\toprule
\textbf{Component} & \textbf{Specification} \\
\midrule
\multicolumn{2}{@{}l}{\textit{Compute node}} \\
Platform            & Dell PowerEdge XE9640 (BIOS 2.11.2) \\
CPU                 & 2\,$\times$ Intel Xeon Platinum 8452Y \\
                    & 36 cores / 72 threads each (144 threads total) \\
Clock               & 0.8--3.2\,GHz \\
Cache               & 3.4\,MiB L1d, 2.3\,MiB L1i, 144\,MiB L2, 135\,MiB L3 \\
NUMA                & 2 nodes \\
System memory       & 1\,TiB DDR5 \\
\midrule
\multicolumn{2}{@{}l}{\textit{Accelerators}} \\
GPU                 & 4\,$\times$ NVIDIA H100 SXM5, 80\,GB HBM3 each \\
Compute capability  & 9.0 (Hopper) \\
TDP                 & 700\,W per GPU \\
Interconnect        & NVLink 4.0, all-to-all (18 links $\times$ 26.6\,GB/s, \\
                    & 900\,GB/s aggregate bidirectional per GPU) \\
\midrule
\multicolumn{2}{@{}l}{\textit{Storage}} \\
System volume       & 447\,GB Dell BOSS-N1 NVMe (ext4, LVM) \\
Data volume         & 4\,$\times$ 1.92\,TB Dell CM7 U.2 NVMe, \\
                    & software RAID, 5.3\,TB XFS \\
\midrule
\multicolumn{2}{@{}l}{\textit{Software}} \\
OS                  & Ubuntu 24.04.5 LTS, kernel 6.8.0-139 (x86\_64) \\
NVIDIA driver       & 580.173.02 (CUDA 13.0 runtime) \\
CUDA toolkit        & 12.0 (nvcc V12.0.140) \\
Compiler            & GCC 13.3.0 \\
Python              & 3.12.3 \\
Deep learning stack & PyTorch 2.11.0, PyTorch Lightning 2.6.1, \\
                    & TorchMetrics 1.9.0, PyTorch Geometric 2.7.0 \\
Numerics            & NumPy 1.26.4, SciPy 1.17.1 \\
\bottomrule
\end{tabular}
\end{table}
 \begin{table}[ht]
\centering
\caption{Model size and training cost. $\Delta$ is the \model{}'s parameter count relative to the baseline (\ie the standard GT used in \cite{qin2025defog}) at the same depth. Time per epoch is the median wall-clock of the training-only epochs (no validation or sampling) logged during each campaign run, excluding the first epoch of every process; times are measured on a single GPU of the platform described in \Cref{tab:system-config}.\label{tab:runtime-simple}}
\small
\begin{tabular}{llrrrrr}
\toprule
 & & \multicolumn{3}{c}{Parameters (M)} & \multicolumn{2}{c}{Time / epoch (s)} \\
\cmidrule(lr){3-5}\cmidrule(lr){6-7}
Dataset & $L$ & Baseline & \model{} & $\Delta$ & Baseline & \model{} \\
\midrule
Planar & 8 & 7.14 & 6.02 & $-15.6\%$ & 0.4 & 0.4 \\
 & 16 & 14.18 & 11.95 & $-15.7\%$ & 0.8 & 0.7 \\
 & 32 & 28.25 & 23.79 & $-15.8\%$ & 1.5 & 1.5 \\
\midrule
SBM & 8 & 7.14 & 6.03 & $-15.6\%$ & 2.9 & 2.6 \\
 & 16 & 14.18 & 11.95 & $-15.7\%$ & 8.3 & 7.2 \\
 & 32 & 28.25 & 23.79 & $-15.8\%$ & 16.5 & 14.5 \\
\midrule
ZINC250k & 8 & 10.93 & 9.68 & $-11.4\%$ & 228 & 219 \\
 & 16 & 21.61 & 19.12 & $-11.5\%$ & 440 & 411 \\
 & 32 & 42.97 & 38.00 & $-11.6\%$ & 1\,329 & 1\,195 \\
\midrule
MOSES & 8 & 10.93 & 9.68 & $-11.4\%$ & 874 & 845 \\
 & 16 & 21.61 & 19.12 & $-11.5\%$ & 1\,660 & 1\,633 \\
 & 32 & --- & 38.00 & --- & --- & 3\,232 \\
\bottomrule
\end{tabular}
\end{table}
 
\subsection{Ablation on the effect of \texorpdfstring{$\gamma_0$}{gamma0}}
\label{app:ablation_gamma}
To isolate the role of the dynamical regime, we vary only the damping $\gamma_0$ at fixed depth $L=32$, keeping architecture and training unchanged. At $\gamma_0=0$ node transport is exactly orthogonal; increasing $\gamma_0$ makes it progressively contractive (\Cref{prop:controlled_dissipation}). \Cref{tab:gamma-sweep-synth} shows that the non-dissipative model is best on both datasets, and that no amount of damping improves on it. On Planar, V.U.N.\ decreases from $100.0\%$ to $95.0\%$ as $\gamma_0$ grows to $4$; on SBM, any damping lowers it from $92.5\%$ to $90.0\%$. Dissipation therefore degrades deep graph generation, which motivates using the non-dissipative regime.
\begin{table}[ht]
\centering
\caption{V.U.N.\ (\%) of \model{} at $L=32$ with damping $\gamma_0$. $\gamma_0=0$ is exactly orthogonal (\ie non-dissipative dynamics), while increasing \(\gamma_0\) induces progressively stronger dissipation. 40 generated graphs.}
\label{tab:gamma-sweep-synth}
\small
\begin{tabular}{lcccc}
\toprule
 & \multicolumn{4}{c}{$\gamma_0$} \\
 & \multicolumn{4}{c}{\footnotesize non-dissipative $\longrightarrow$ dissipative} \\
\cmidrule(lr){2-5}
 & 0 & 0.5 & 1 & 4 \\
\midrule
Planar & \textbf{100.0} & 97.5 & 97.5 & 95.0 \\
SBM & \textbf{92.5} & 90.0 & 90.0 & 90.0 \\
\bottomrule
\end{tabular}
\end{table}
  
\end{document}